\documentclass[letterpaper, conference]{ieeeconf}
\IEEEoverridecommandlockouts
\usepackage{cite}

\def\BibTeX{{\rm B\kern-.05em{\sc i\kern-.025em b}\kern-.08em
    T\kern-.1667em\lower.7ex\hbox{E}\kern-.125emX}}
    
  \usepackage{amsmath,amssymb,amsfonts,amsthm} 
\usepackage{mathtools}
\usepackage{commath}
\usepackage{algorithmic}
\usepackage[ruled,vlined]{algorithm2e}
\usepackage{graphicx}
\usepackage{textcomp}
\usepackage[table]{xcolor}

\usepackage{framed}
\definecolor{shadecolor}{rgb}{.9,.9,.9}
\usepackage{tcolorbox}
\usepackage{siunitx}
\usepackage{textgreek}
\usepackage{url}
\usepackage{booktabs, tabularx}
\usepackage{multirow}
\usepackage{stackengine}
\usepackage[caption=false, font=footnotesize]{subfig}

\usepackage{tikz}
\usetikzlibrary{quotes,angles}
\usetikzlibrary{calc}
\usetikzlibrary{decorations.pathmorphing}

\usepackage{soul}
\usepackage{todonotes}
\usepackage{tablefootnote}

\usepackage[absolute]{textpos}     % used to position the copyright statement
\usepackage[hidelinks]{hyperref}   % used to put url link in copyright statement

\title{\LARGE\bf Amplitude Control for Parallel Lattices of Docked Modboats 
}

\author{Gedaliah Knizhnik and Mark Yim% <-this % stops a space
\thanks{The authors are with the GRASP Laboratory, University of Pensylvannia, Philadelphia, PA 19104. 
        {\tt\footnotesize knizhnik@seas.upenn.edu}}%
}

\newcommand{\copyrightstatement}{
    \begin{textblock*}{5.7in}(0.25in,0.25in) % {box width}(leftposition, rightposition)

        \noindent
        \footnotesize
        This accepted article to ICRA is made available by the authors in compliance with IEEE policy.

        \noindent
        Please find the final, published version in IEEE Xplore, DOI: \href{https://doi.org/10.1109/ICRA46639.2022.9812381}{\textcolor{blue}{10.1109/ICRA46639.2022.9812381}}.
        
    \end{textblock*}

    \begin{textblock*}{5.7in}[0,1](0.25in,10.85in) % {box width}(leftposition, rightposition)

        \noindent
        \footnotesize \scriptsize
        \copyright 2022 IEEE. Personal use of this material is permitted.
        Permission from IEEE must be obtained for all other uses, in any current or future media, including reprinting/republishing this material for advertising or promotional purposes, creating new collective works, for resale or redistribution to servers or lists, or reuse of any copyrighted component of this work in other works.
    \end{textblock*}
}

\DeclareMathOperator{\sign}{sgn}

\newtheorem{assumption}{Assumption}%[section]
\newtheorem{theorem}{Theorem}

\begin{document}
%\bstctlcite{IEEEexample:BSTcontrol}
\bstctlcite{MyBSTcontrol} % Removes URLs from citations
\copyrightstatement                    % include copyright statement

\maketitle

\begin{abstract}
The Modboat is a low-cost, underactuated, modular robot capable of surface swimming. It is able to swim individually, dock to other Modboats, and undock from them using only a single motor and two passive flippers. Undocking without additional actuation is achieved by causing intentional self-collision between the tails of neighboring modules; this becomes a challenge when group swimming as one connected component is desirable. In this work, we develop a control strategy to allow parallel lattices of Modboats to swim as a single unit, which conventionally requires holonomic modules. We show that the control strategy is guaranteed to avoid unintentional undocking and minimizes internal forces within the lattice. Experimental verification shows that the controller performs well and is consistent for lattices of various sizes. Controllability is maintained while swimming, but pure yaw control causes lateral movement that cannot be counteracted by the presented framework.
\end{abstract}

%%%%%%%%%%%%%%%%%%%%%%%%%%%%
%%%%%%%%%%%%%%%%%%%%%%%%%%%%
%%%%%%%%%%%%%%%%%%%%%%%%%%%%

\section{Introduction} \label{sec:intro}

Aquatic systems that can dock, undock and reconfigure are of interest to researchers; they have significant potential to facilitate ocean research and infrastructure by providing mobile platforms to land helicopters or drones, building bridges for larger vehicles\cite{Paulos2015}, or forming ocean-going manipulators. They can adapt to changing flow conditions or take precise measurements at small spatial scales. Nevertheless, work on such reconfigurable systems is not widespread. A small number of projects have used reconfiguration to multiplex swimming modes, such as AMOUR \cite{Vasilescu2010AMOURPayloads}, and ANGELS~\cite{Mintchev2014}\cite{Mintchev2014TowardsRobots}. Both of these systems use robots capable of almost holonomic motion in 3D space, but only demonstrate minimal reconfiguration in 1D.

Extensive reconfiguration, however, has only been demonstrated --- to the best of the authors' knowledge --- in two cases. The TEMP project built holonomic rectangular modules that could dock in a brickwork pattern and form a bridge~\cite{Paulos2015} or assemble into arbitrary structures~\cite{Seo2013AssemblyRobots}\cite{Seo2016AssemblyModules}. The structures could actively resist deformation due to waves~\cite{OHara2014} and the undocked modules could move in formation, but movement of the entire structure was not considered. The Roboats project has designed similar rectangular modules capable of docking and reconfiguration, which are meant to form floating platforms in the canals of Amsterdam~\cite{Wang2018DesignVehicle}\cite{Wang2020RoboatEnvironments}. Roboats can swim as a group, and can move using local information as long as the structure is symmetric~\cite{Wang2020DistributedVessels} or through a learning approach for more arbitrary structures~\cite{Park2019CoordinatedApproach}. But both TEMP and Roboats require expensive modules capable of holonomic motion and rely extensively on this ability for effective motion control.

\begin{figure}[t]
    \centering
    \includegraphics[page=5, width=\linewidth]{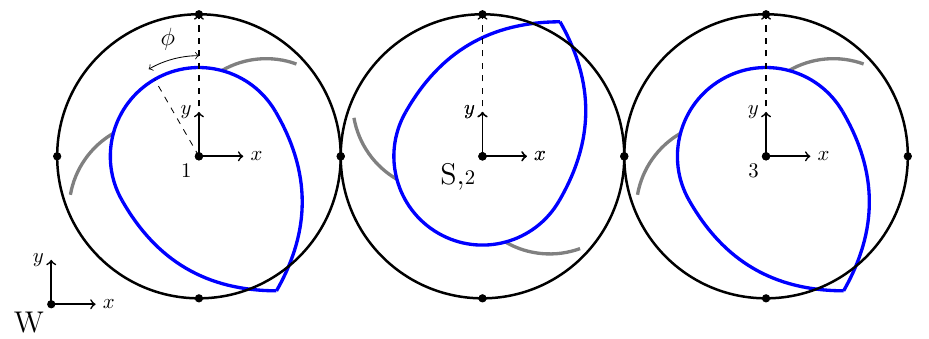}
    \caption{Parallel lattice of three docked Modboats; for each boat, the top body is shown in black, while the bottom body/tail is in blue and flippers are in gray. Magnetic docking points are shown at the cardinal points of each boat, and the motor angle $\phi$ is the angle of the bottom body \textit{relative} to the top body. Frame $W$ is the fixed world frame, and frame $S$ is the body-fixed frame at the COM, which coincides with the body-fixed frame of boat $2$. The individual boat frames are aligned with the top bodies.}
    \label{fig:diagram}
\end{figure}

The Modboat, introduced in prior work~\cite{Knizhnik2021}\cite{modboatsOnline}, is an inexpensive, underactuated planar modular robot. Passive flippers convert the oscillation of its two concentric bodies (see Fig.~{\ref{fig:diagram}} and~\cite{modboatsOnline}) into propulsion and steering, with symmetric oscillations generating propulsion and asymmetric ones generating torques. We have also shown that the Modboat is capable of docking to and undocking from other modules to form a 4-connected structured arrangement~\cite{Knizhnik2021}, which we call a \textit{lattice}. Control of this lattice, however, has not been attempted before this work, and is a significant challenge for several reasons:

\begin{enumerate}
    \item Aquatic lattice control to date (e.g Roboats and TEMP) has required holonomic modules, but Modboat modules are underactuated. Force can be created in any direction on average over a long time-scale, but not sufficiently quickly to use holonomic approaches for control.
    \item Modboat modules are connected via permanent magnets to enable latch-less docking/undocking~\cite{Knizhnik2021}. The magnetic bond is relatively weak, however, and the magnets provide effectively no resistance to rotation. Angular acceleration between modules can cause them to undock if unchecked, as can undue lateral forces.
    \item Modboat modules use intentional self-collision between their tails (see Fig.~\ref{fig:diagram} and~\cite{Knizhnik2021}) to undock without additional actuation. Swimming as a unit, however, requires avoiding unintentional undocking, which creates a complex constraint on allowable control inputs.
\end{enumerate}

These challenges require significant attention to internal forces, since we cannot create arbitrary forces to balance the lattice and must account for possible changes in thrust direction as modules shift on their docks. Our long-term goal is to control arbitrary configurations of docked Modboats, but in this work we focus on \textit{parallel lattices} (i.e. where all modules are placed on a line perpendicular to the desired ``front'' direction, such as shown in Fig.~\ref{fig:diagram}). The contribution of this work is thus a centralized control strategy for parallel lattices of Modboats that (1) functions for underactuated modules capable of thrust along a single axis, (2) minimizes internal forces within the lattice, and (3) guarantees no unintentional undocking of modules.

The rest of this work is organized as follows. In Sec.~\ref{sec:selfCollision} we discuss a strategy for avoiding unintentional undocking, and develop a control formulation that follows this strategy in Sec.~\ref{sec:control}. Sec.~\ref{sec:experiments} presents experimental verification of controller performance, which is discussed in detail in Sec.~\ref{sec:discussion}.

%%%%%%%%%%%%%%%%%%%%%%%%%%%%
%%%%%%%%%%%%%%%%%%%%%%%%%%%%
%%%%%%%%%%%%%%%%%%%%%%%%%%%%

\section{Avoiding Unintentional Undocking} \label{sec:selfCollision}

As shown in Fig.~\ref{fig:diagram}, the bottom body of the Modboat (blue) and flippers (gray) are fully contained within the footprint of the top body (black) except for the tip of the tail. This ensures that the flippers of neighboring modules cannot mechanically interact, but the tails can be used to undock from neighboring modules by bringing them into contact (see Fig.~\ref{fig:diagram})~\cite{Knizhnik2021}, which is essentially self-collision within the lattice. This is advantageous because docking and undocking can be performed without additional actuation, but it introduces a complex constraint during group control while swimming as a single unit. 
\begin{figure}[t]
    \centering
    \subfloat[\label{fig:freqMod}]{\includegraphics[width=\linewidth, trim={0 0 0 0.5cm}]{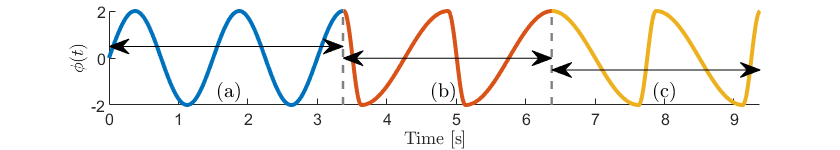}
       }
          \hfill
    \subfloat[\label{fig:pauseMod}]{\includegraphics[width=\linewidth, trim={0 0 0 0.5cm}]{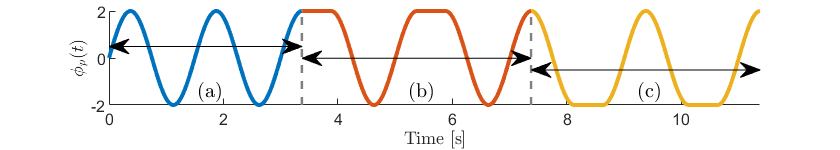}
       }
    \caption{Examples of (a) differential thrust and (b) inertial control waveforms used for single boat control as defined in prior work~\cite{Knizhnik2020a}, with colors and arrows indicating regions in which the control parameters change. Neither is suitable for group control because phase differences between neighbors will occur.}
    \label{fig:prevWaves}
\end{figure}

We cannot use either of the single-module control modes defined in our prior work~\cite{Knizhnik2020a} and shown in Fig.~\ref{fig:prevWaves} when the self-collision constraint is considered. Differential thrust, shown in Fig.~\ref{fig:freqMod}, uses frequency modulation of the motor angle to steer; inertial control, shown in Fig.~\ref{fig:pauseMod}, uses pauses in the motor angle waveform to steer. Both control methods quickly lead to phase differences between the tail angles $\phi$ of neighboring boats (see Fig.~\ref{fig:diagram}) such that no guarantees can be made on avoiding self-collision. We also wish to avoid simply limiting the amplitude, which would enforce an overly restrictive thrust limit. 

Consider two neighboring Modboats, each executing the waveform shown in~\eqref{eq:waveform} with amplitude $A$, centerline $\phi_0$, and angular frequency $\omega$, where the subscripts $i$ and $k$ indicate the boat and the cycle respectively. The parameters of~\eqref{eq:waveform} are chosen at the beginning of each cycle and executed for a complete period. Considering the complex geometry of the tail, characterizing the set of parameters $(\phi_0,A,\omega)$ for neighboring boats $i$ and $j$ that would lead to self-collision anywhere in the lattice is a complex task. We can instead consider a conservative approach that guarantees no unintentional undocking while still maintaining controllability for the lattice as a whole.
\begin{equation} \label{eq:waveform}
    \phi_i(t) = (\phi_0)_{i,k} + A_{i,k} \cos{(\omega t)}
\end{equation}

\begin{assumption}[Phase Lock] \label{asp:phaseLock}
All boats are \textbf{in phase} with one another for all time, i.e there is no phase offset in~\eqref{eq:waveform} and $\omega$ is the same for all boats\footnote{The angular frequency $\omega$ can change between cycles $k$, although without loss of generality we assume it does not.}. Control decisions are made simultaneously for all boats in the configuration at the end of each cycle.
\end{assumption}

\begin{assumption}[Forward/Reverse] \label{asp:forward}
All boats can choose the centerline of rotation $(\phi_0)_i$ to be either $0$ (forward) or $\pi$ (reverse) in each cycle. No other angles are allowed.
\end{assumption}

Under assumption~\ref{asp:forward} we can redefine~\eqref{eq:waveform} into~\eqref{eq:waveformNew}. This minimizes the discontinuity in the wave when transitioning between choices of $(\phi_0)_i$, since $\cos{\left ((\phi_0)_{i} \right )} = \pm 1$.
\begin{equation}\label{eq:waveformNew}
        \phi_i(t) = (\phi_0)_{i,k} + A_{i,k} \cos{(\omega t)}\cos{\left ((\phi_0)_{i,k} \right )}
\end{equation}

Assumptions~\ref{asp:phaseLock} and ~\ref{asp:forward} ensure that each boat maintains its ability to create both forward and reverse thrust, but guarantee that collisions can not occur, as in Theorem~\ref{th:selfCollision}.

\begin{figure}[t]
    \centering
    \includegraphics[page=3, width=\linewidth]{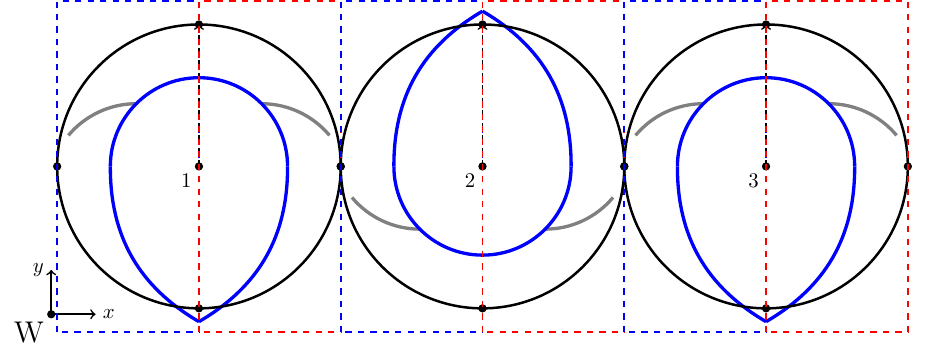}
    \caption{Parallel lattice of three docked Modboats, with red-blue tiling used in proof for Theorem~\ref{th:selfCollision}. Modboat $2$ is shown in the reverse paddling configuration ($\phi = \phi_0 = \pi$ in~\eqref{eq:waveform}), while the rest are in the forward configuration ($\phi = \phi_0=0$), shown at $t = 0$.}
    \label{fig:diagramForProof}
\end{figure}

\begin{theorem}[No Unintentional Undocking]\label{th:selfCollision}
When swimming under assumptions~\ref{asp:phaseLock} and ~\ref{asp:forward}, and with waveform given by~\eqref{eq:waveformNew}, we can guarantee that no unintentional undocking events will occur within a parallel lattice.
\end{theorem}
\begin{proof}
Construct a lattice of neighboring red and blue regions, as shown in Fig.~\ref{fig:diagramForProof}. For any lattice length, it is trivial to periodically tile. Consider some time $t_0$ when $\cos{(\omega t_0)}=1$. Then $\phi_i = (\phi_0)_i + A\cos{\left((\phi_0)_{i} \right )}$ $\forall i$, and  
\begin{equation*}
\phi_i = \begin{cases}
(\phi_0)_i & t = t_0 + T/4 \\
(\phi_0)_i - A\cos{\left((\phi_0)_{i} \right )} & t = t_0 + T/2 \\
(\phi_0)_i & t = t_0 + 3T/4 \\
(\phi_0)_i + A\cos{\left((\phi_0)_{i} \right )} & t = t_0 + T
\end{cases} \quad \forall i
\end{equation*}

Thus the tail segments are \textbf{all} in red regions\footnote{The tail tip will eventually enter the neighboring boat's blue region. For the purposes of the proof, we consider a slightly interior point that remains within the red region.} when $t \in (t_0, t_0 + T/4) \cap (t_0 + 3T/4,t_0 + T)$, and blue regions when $t \in (t_0 + T/4, t_0 + 3T/4)$. Thus, at all times all tails occupy identically colored regions. By construction, no neighboring regions share a color, so no collisions are possible. 
\end{proof}

Note that waveform~\eqref{eq:waveformNew} and Theorem~\ref{th:selfCollision} rely exclusively on internal motor position. Thus their guarantees remain in place despite external disturbances, unless water conditions cause large enough forces to detach the magnets and undock modules. Dealing with disturbances of such magnitude is outside the scope of this work.

%%%%%%%%%%%%%%%%%%%%%%%%%%%%
%%%%%%%%%%%%%%%%%%%%%%%%%%%%
%%%%%%%%%%%%%%%%%%%%%%%%%%%%

\section{Controlling the lattice} \label{sec:control}

Consider a Modboat lattice as shown in Fig.~\ref{fig:diagram}. A surge force $F_y$ and torque $\tau$ applied to the lattice are sufficient to control it in the plane, and for $N\geq 2$ boats these can be arbitrarily set by the modules (within actuation limits). For $N>2$, however, the lattice is overdetermined. Thus the challenge is to determine what forces each Modboat module must apply to steer the lattice.

Let frame $W$ denote the fixed world frame, and let frame $S$ represent the body-fixed structural frame centered at the center of mass (COM) of the lattice. For each boat $i \in [1,\hdots,N]$, frame $B_i$ is then attached to its own COM and aligned with the top body orientation. Without loss of generality, we can define $B_i$ $\forall i$ to be aligned with $S$ so that only a translation is needed to transform between them. We define $x_i$ to be the distance along $\hat{x}_S$ from $S$ to $B_i$.

In the following we assume a planar model (i.e. the water is flat and no waves are present), and that there are no hydrodynamic interactions between boats in the lattice. The extent to which this latter assumption is valid is left to future work, but the controller performance presented in Sec.~\ref{sec:experiments} shows that it is not unreasonable. 

%%%%%%%%%%%%%%%%%%%%%%%%%%%%
%%%%%%%%%%%%%%%%%%%%%%%%%%%%
%%%%%%%%%%%%%%%%%%%%%%%%%%%%

\subsection{Dynamic Model}

When executing the waveform given in~\eqref{eq:waveformNew}, each Modboat can be considered to be producing thrust $f_i \in [-f_{max},f_{max}]$ \textit{on average} along the direction given by $(\phi_0)_i$ over a full cycle of length $T$ (where $T$ is the period corresponding to the angular frequency $\omega$). Under assumption~\ref{asp:forward} each force $f_i$ acts along the $\hat{y}_S$ (surge) axis.

\begin{figure}[t]
    \centering
    \includegraphics[width=\linewidth]{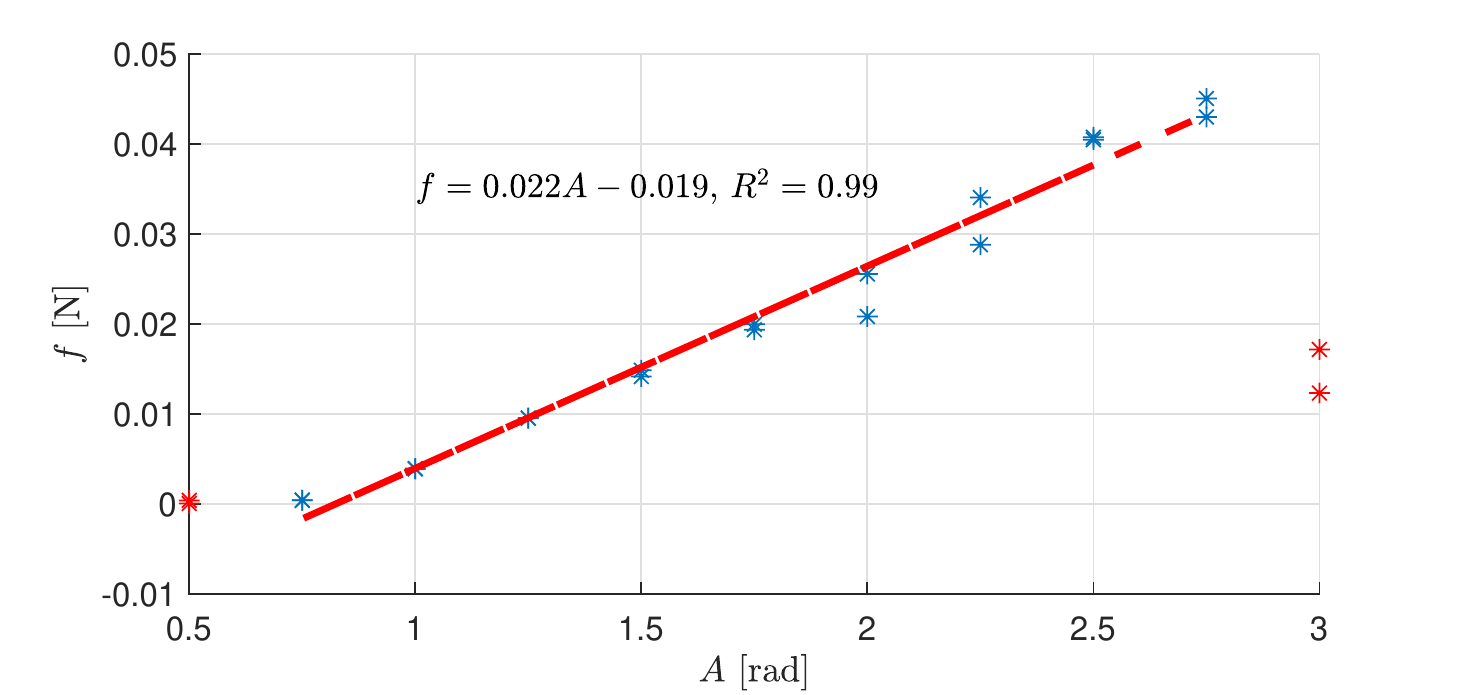}
    \caption{The experimentally determined force vs. amplitude curve for a period of oscillation $T=1.5~\si{s}$. Data points are shown in blue, and red stars indicate poor force generation due to incomplete flipper activation (at low amplitudes) or significant reverse thrust (at high amplitudes).}
    \label{fig:forceFromAmp}
\end{figure}

We define $v_y$ and $a_y$ as the velocity and acceleration along the surge axis, respectively. We can then write the equations of motion for the lattice as in~\eqref{eq:eom-y} and~\eqref{eq:eom-yaw}, where $m = \sum{m_i}$ is the total mass of the lattice, and $I$ is the total moment of inertia about the $z$ axis. The angular acceleration and velocity of the lattice are given by $\alpha$ and $\Omega$, respectively, and $C_L$ and $C_R$ are drag coefficients in the forward and yaw directions, respectively. Note that the equations are written continuously in time since the average thrust is used.
\begin{align}
    m a_y &= \sum_i f_i \hphantom{x_i} - C_L \abs{v_y} v_y  \label{eq:eom-y}\\
    I \alpha &= \sum_i f_{i}x_{i} - C_R \abs{\Omega} \Omega \label{eq:eom-yaw}
\end{align}
Our desired controller will track a commanded surge velocity $v_c$. Since this velocity will remain generally consistent, we can consider forward motion as occurring in steady state, in which thrust balances drag on average. Thus we can take $a_y = 0$. The yaw angle is allowed to change more frequently, so we do not assume steady state in yaw.

Let $\vec{f}\in\mathbb{R}^N$ represent the vector of boat forces $\vec{f} = [ \begin{matrix} f_1 & \hdots & f_N \end{matrix} ]^T$. We can then define a structural matrix $P$ as in \eqref{eq:structuralMatrix}, and~\eqref{eq:eom-y} and~\eqref{eq:eom-yaw} can be rewritten in vector form~\eqref{eq:eom-mat1}.
\begin{equation} \label{eq:structuralMatrix}
    P = \begin{bmatrix} 1 & 1 & \hdots & 1 & 1 \\ x_1 & x_2 & \hdots & x_{N-1} & x_N \end{bmatrix}
\end{equation}
\begin{equation} \label{eq:eom-mat1}
    P\vec{f} = \begin{bmatrix}  C_L\abs{v_c}v_c \\ I\alpha + C_R\abs{\Omega}\Omega \end{bmatrix}
\end{equation}

In~\eqref{eq:eom-mat1}, $P\vec{f}\in\mathbb{R}^2$ represents the surge force and yaw torque that steers the lattice as a whole. Applying suitable forward force will create surge velocity, while torquing the lattice allows it to maintain a desired yaw and heading. We can then use the Moore-Penrose inverse $P^+ = P^T(PP^T)^{-1}$ to find the forces desired from each boat, as in~\eqref{eq:forces}.
\begin{equation} \label{eq:forces}
    \vec{f} = P^+\begin{bmatrix}  C_L\abs{v_c}v_c \\ I\alpha + C_R\abs{\Omega}\Omega \end{bmatrix}
\end{equation}

\begin{figure}[t]
    \centering
    \subfloat[\label{fig:linearDrag}]{\includegraphics[width=\linewidth, trim={0 0 0 0.5cm}]{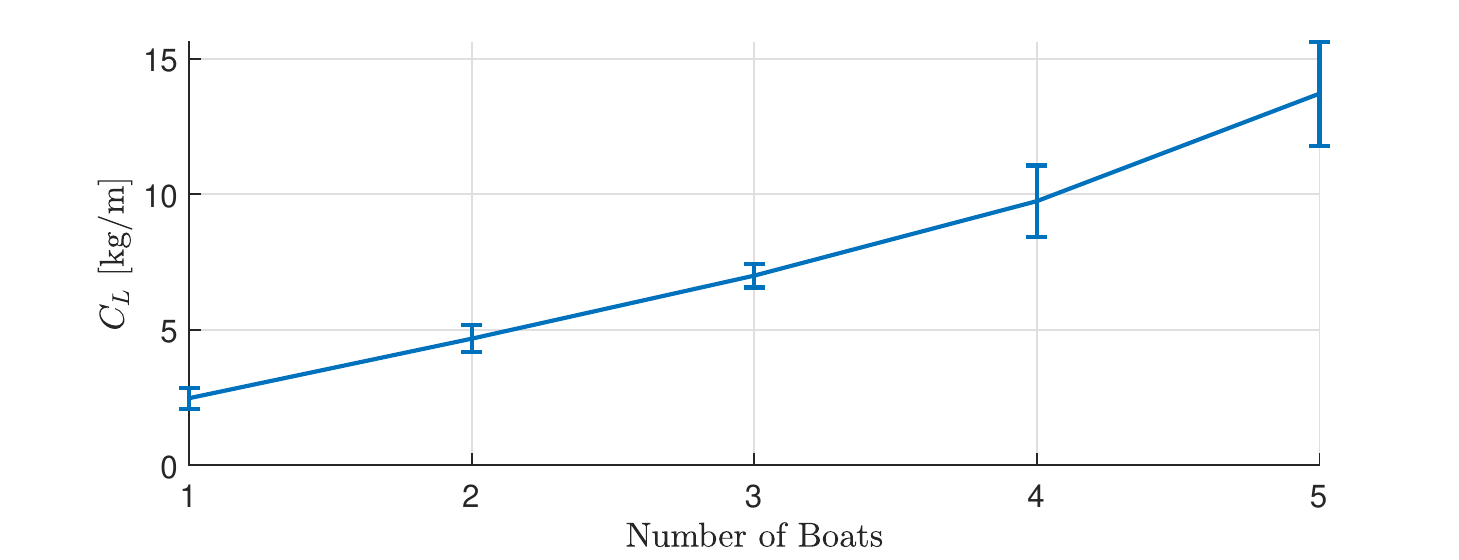}
       }
          \hfill
    \subfloat[\label{fig:angularDrag}]{\includegraphics[width=\linewidth, trim={0 0 0 0.5cm}]{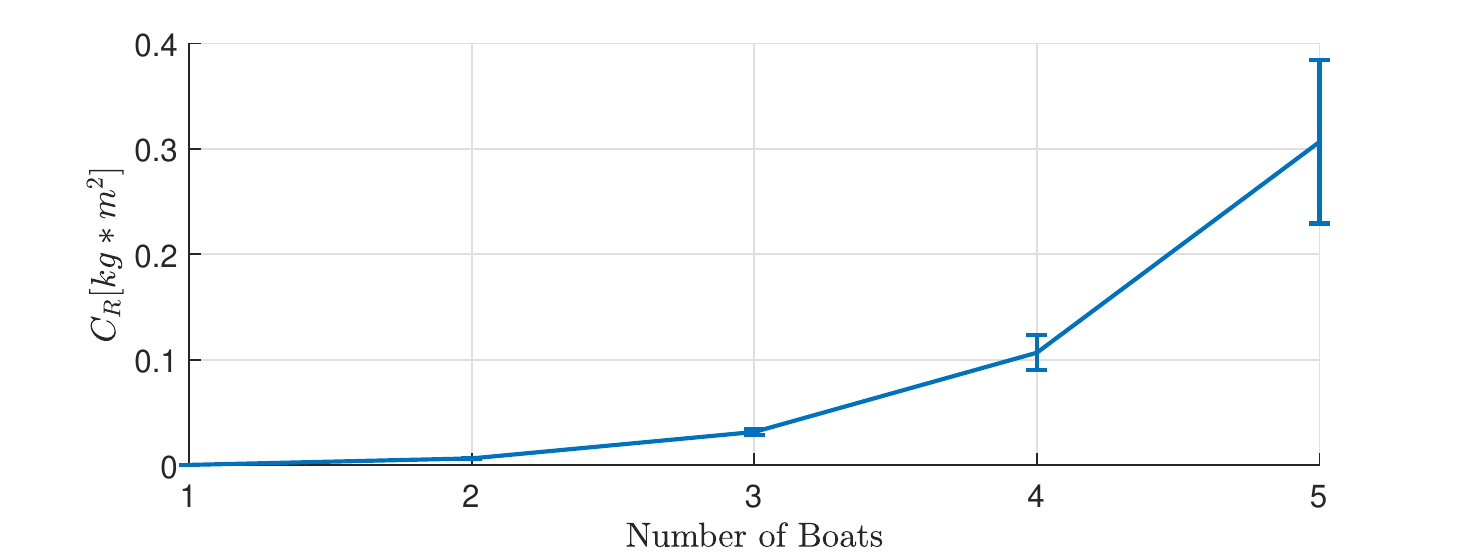}
        }
    \caption{Experimentally determined drag coefficients (a) $C_L$ and (b) $C_R$ vs. number of boats in a parallel configuration, as defined in~\eqref{eq:eom-y} and~\eqref{eq:eom-yaw}. The linear drag coefficient $C_L$ is roughly linear in the number of boats, while the angular drag coefficient $C_R$ is roughly quadratic. Error bars indicate one standard deviation.}
    \label{fig:dragCoeffs}
\end{figure}

Using~\eqref{eq:forces} to distribute forces among the modules results in a linear distribution along the $x$-axis of the lattice; a similar approach was used by Gabrich to distribute forces in a lattice of docked quadrotors~\cite{Gabrich2020ModQuad-DoF:Quadrotors}. Other distributions are possible, but a linear distribution most closely matches the internal dynamics that would be observed if the lattice were a single rigid body. As discussed in Sec.~\ref{sec:intro}, Modboat modules are docked using passive magnets that allow the units to rotate somewhat relative to one another, and significant intra-lattice forces can cause enough rotation to undock. Maintaining a rigid-body force distribution minimizes such forces and oscillation between neighboring modules.

%%%%%%%%%%%%%%%%%%%%%%%%%%%%
%%%%%%%%%%%%%%%%%%%%%%%%%%%%
%%%%%%%%%%%%%%%%%%%%%%%%%%%%

\subsection{Control Input}

The control structure given in~\eqref{eq:forces} can track a desired surge velocity $v_d$ and yaw angle $\theta_d$. Since forward motion occurs in steady-state, it should be sufficient to use the desired velocity $v_d$ as the commanded velocity $v_c$ in~\eqref{eq:forces}. In practice it is observed that this is not sufficient, however, and using $v_d$ directly results in very slow observed velocity $v_{obs}$. Instead, an artificial linear acceleration is calculated in~\eqref{eq:errVel} and~\eqref{eq:pidVel} and used to integrate the commanded velocity as in~\eqref{eq:vCommand}. This discrepancy between $v_c$ and $v_d$ may be a result of hydrodynamic interactions between neighbors, but a full consideration is left to future work.
\begin{align}
    e_{v} &= v_{d} - v_{obs} \label{eq:errVel}\\ 
    \tilde{a}_y &= K_{pv} e_{v} + K_{dv} \frac{de_{v}}{dt} \label{eq:pidVel} \\
    v_c &= v_d + \int_{0}^t \tilde{a}_y d\tau \label{eq:vCommand}
\end{align}

The yaw angle can be commanded by a standard PD control loop on the angular acceleration, as given in~\eqref{eq:errYaw} and \eqref{eq:pidYaw}, where $\Theta$ and $\Theta_{des}$ are the observed and desired yaw angles for the lattice, respectively. 
\begin{align}
    e_{\Theta} &= \Theta_{des} - \Theta \label{eq:errYaw}\\ 
    \alpha &= K_{p\Theta} e_{\Theta} + K_{d\Theta} \frac{de_{\Theta}}{dt} \label{eq:pidYaw}
\end{align}

%%%%%%%%%%%%%%%%%%%%%%%%%%%%
%%%%%%%%%%%%%%%%%%%%%%%%%%%%
%%%%%%%%%%%%%%%%%%%%%%%%%%%%

\subsection{Input Values} \label{sec:amplitude}

Eq.~\eqref{eq:forces} gives the force distribution for each module and is valid for any single-thruster robot, but the required input for the Modboats is the choice of centerline $(\phi_0)_i$ and oscillation amplitude $A_i$. Theorem~\ref{th:selfCollision} guarantees there will be no unintentional undocking, so we can safely map the desired forces to those values. Since there is a practical maximum to the force the Modboats can generate, $f_i$ is first capped by $f_{max}$. Leveraging the symmetry of the Modboat the centerline for each boat $(\phi_0)_i$ is determined by $\sign{(f_i)}$, and the amplitude is determined by $\abs{f_i}$.

The mapping $f_i = f(A_i)$ can be determined experimentally as follows. We created a parallel lattice of three Modboats and commanded only the center module to execute the waveform in~\eqref{eq:waveformNew} with $T = 1.5~\si{s}$ and $A\in[0.5,3.0]~\si{rad}$. The period-wise surge velocity $v_y$ was calculated for the lattice as in~\eqref{eq:periodwiseVelocity}, and the steady-state velocity was converted to force via the quadratic drag relationship $f_i = C_L v_y^2$ and the drag coefficient as calculated below.
\begin{equation} \label{eq:periodwiseVelocity}
    \vec{v}_T(t) = \frac{\vec{x}(t) - \vec{x}(t-T)}{T}
\end{equation}

Fig.~\ref{fig:forceFromAmp} shows the resulting mapping $f = f(A)$, which is linear within the range $A\in[0.75,2.75]~\si{rad}$. Below $0.75~\si{rad}$ (with $T=1.5~\si{s}$) the flippers do not fully open, so negligible thrust is produced. Above $2.75~\si{rad}$, the tail rotates enough to produce significant reverse thrust during a portion of the cycle. Thus we intentionally limit the maximum allowable amplitude to $2.75~\si{rad}$\footnote{In practice the maximum amplitude used experimentally is $2.5~\si{rad}$, since higher amplitudes are found to cause the lattice to shake internally.}.

The mass of an individual boat $m_i$ can be measured, and the moment of inertia for an individual boat $I_i$ can be calculated from its Solidworks model files. The total moment of inertia $I$ can then be calculated via the parallel axis theorem for any lattice. The drag coefficients $C_L$ and $C_R$ can be experimentally calculated; a linear(angular) impulse is delivered to the lattice, and the resulting linear(angular) velocity is tracked. A nonlinear least-squares fit to a quadratic drag model then gives $C_L$ and $C_R$. The resulting drag coefficients are shown in Fig.~\ref{fig:dragCoeffs}, where $C_L$ is shown to be roughly linear and $C_R$ is roughly quadratic.

\begin{table}[t]
    \centering
    \caption{Mass and inertia for each parallel boat configuration.}
    \begin{tabular}{rl|c|c|c|c|c}
    \toprule
    $N$ & Boats & 1 & 2 & 3 & 4 & 5 \\ \midrule
    $M$ & $[\si{kg}]$   &  0.66 & 1.32 & 1.98 & 2.64 & 3.30\\
    $I$ & $[\si{g*m^2}]$ &  2.05 & 11.8 & 36.8 & 84.8 & 164 \\
    $C_L$ & $[\si{kg/m}]$ & 2.48 & 4.67 & 7.00 & 9.75 & 13.7 \\
    $C_R$ & $[\si{g*m^2}]$ & 0.40 & 6.50 & 32.0 & 107 & 307 \\ \bottomrule
    \end{tabular}
    \label{tab:parameters}
\end{table}

\begin{table}[t]
    \centering
    \caption{Controller performance across all configurations as IQR. RMS error is after the rise time, from $0.3-0.9$ of the step.}
    \begin{tabular}{rl|c|c}
    \toprule
     & & Swim/Yaw only & Swim and Yaw  \\ \midrule
    $v_y$ RMS & $[\si{cm/s}]$ & $[0.13,0.38]\hphantom{^\dagger}$ & $[0.23\hphantom{0},0.42]\hphantom{^\dagger}$ \\
    $v_y$ Rise & $[\si{s}]$   & $[4.6\hphantom{0}, 6.4\hphantom{0}]\hphantom{^\dagger}$ & $[4.5\hphantom{00},8.0\hphantom{0}]\hphantom{^\dagger}$ \\
    $\theta$ RMS & $[\si{rad}]$ & $[0.10,0.24]\hphantom{^\dagger}$ & $[0.065,0.15]^\dagger$\\
    $\pi/2$ step $\theta$ Rise & $[\si{s}]$ & $[3.7\hphantom{0},5.0\hphantom{0}]\hphantom{^\dagger}$  & N/A \\
    $\pi$ step $\theta$ Rise &  $[\si{s}]$ & $[3.9\hphantom{0}, 6.9\hphantom{0}]^\dagger$ & $[5.7\hphantom{00},13.3]\hphantom{^\dagger}$ \\ \bottomrule \vspace{-2ex} \\  
    \multicolumn{4}{l}{$^\dagger$ Statically significant improvement over the other case in each row.} \vspace{-4ex}
    \end{tabular}
    \label{tab:performanceQ}
\end{table}

%%%%%%%%%%%%%%%%%%%%%%%%%%%%
%%%%%%%%%%%%%%%%%%%%%%%%%%%%
%%%%%%%%%%%%%%%%%%%%%%%%%%%%

\section{Experiments} \label{sec:experiments}

\begin{figure}[t]
    \centering
    \includegraphics[width=\linewidth]{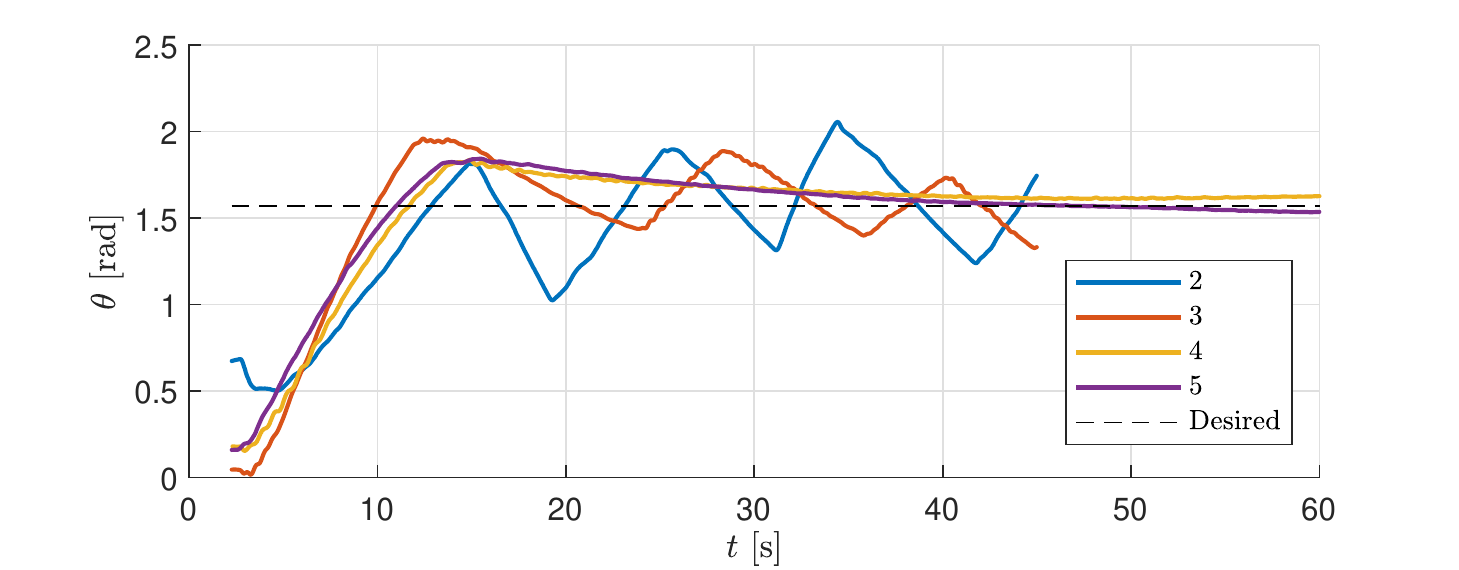}
    \caption{Yaw response to a step input of $\pi/2~\si{rad}$ for configurations of 2--5 boats. The controller provides comparable rise time performance regardless of the number of boats in the lattice, but oscillation amplitudes decrease as the lattice grows.}
    \label{fig:yawResults}
\end{figure}

\begin{figure}[t]
    \centering
    \includegraphics[width=\linewidth]{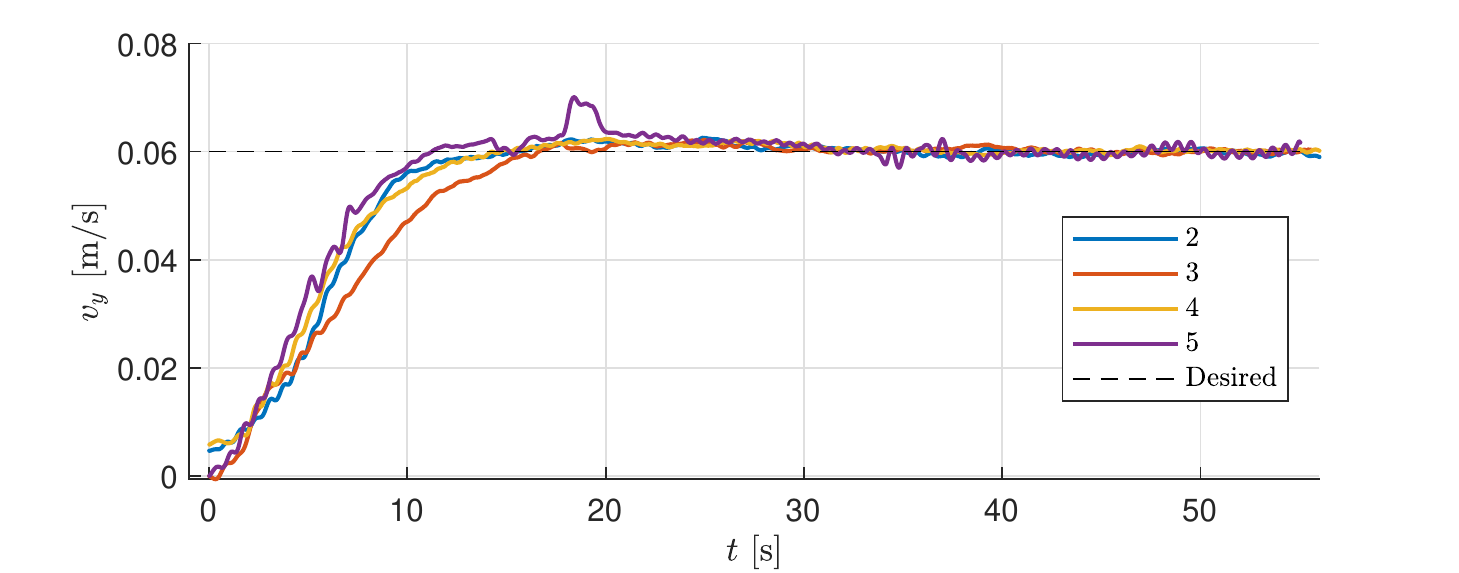}
    \caption{Velocity response to a step input of $6~\si{cm/s}$ for configurations of 2--5 boats. The controller provides comparable performance regardless of the number of boats in the lattice.}
    \label{fig:velResults}
\end{figure}

Experiments were performed in a $4.5~\si{m} \times 3.0~\si{m} \times 1.2 ~\si{m}$ tank of still water. An OptiTrack motion capture system provided planar position, orientation, and velocity data at $120~\si{Hz}$, and a MATLAB script calculated control inputs and transmitted waveform parameters to the various boats over WiFi. All experiments were performed for configurations of 2--5 boats aligned in parallel (as in Fig.~\ref{fig:diagram}), and their measured parameters are presented in Table~\ref{tab:parameters}.

The drag coefficients for parallel configurations of 1--5 boats\footnote{Although our controller is degenerate for a single boat, we include its parameters in Table~\ref{tab:parameters} for completeness.} were evaluated by applying impulses and performing a least-squares fit of a quadratic drag model to the resulting velocity data. The resulting coefficients are shown in Fig.~\ref{fig:dragCoeffs}, where the linear drag is roughly linear in the number of boats, and the angular drag is roughly quadratic.

Given the drag coefficients and the thrust-amplitude map in Fig.~\ref{fig:forceFromAmp}, we evaluated the yaw and velocity tracking abilities of the controller. Yaw control was evaluated by commanding $v_y=0$ and a step input for $\theta_{des}$; results are shown in Fig.~\ref{fig:yawResults}, and a sample trajectory is shown in Fig.~\ref{fig:sidewaysAfterYaw}. Velocity tracking was similarly evaluated by commanding a step input forward velocity of either $4~\si{cm/s}$ or $6~\si{cm/s}$ while not commanding a desired yaw; representative results are shown in Fig.~\ref{fig:velResults}. Finally, the combined controller was evaluating by commanding a desired yaw $\theta_{des}$ and velocity for $45~\si{s}$ and then changing the yaw to $\theta_{des}+\pi$ (i.e. a $180^\circ$ turn) with the same velocity for another $45~\si{s}$. A sample such trajectory is shown in in Fig.~\ref{fig:trajectory}. 

To evaluate the stability of the controller for various configurations, we also considered how the system would behave under mismatched drag coefficients. To consider the most extreme example of this, the five-boat parallel configuration evaluations were repeated with the assigned drag coefficients as measured for a two-boat configuration, and the results are shown in Fig.~\ref{fig:mismatched}.

\begin{figure}[t]
    \centering
    \includegraphics[width = \linewidth, trim={0 0 0 0.5cm}]{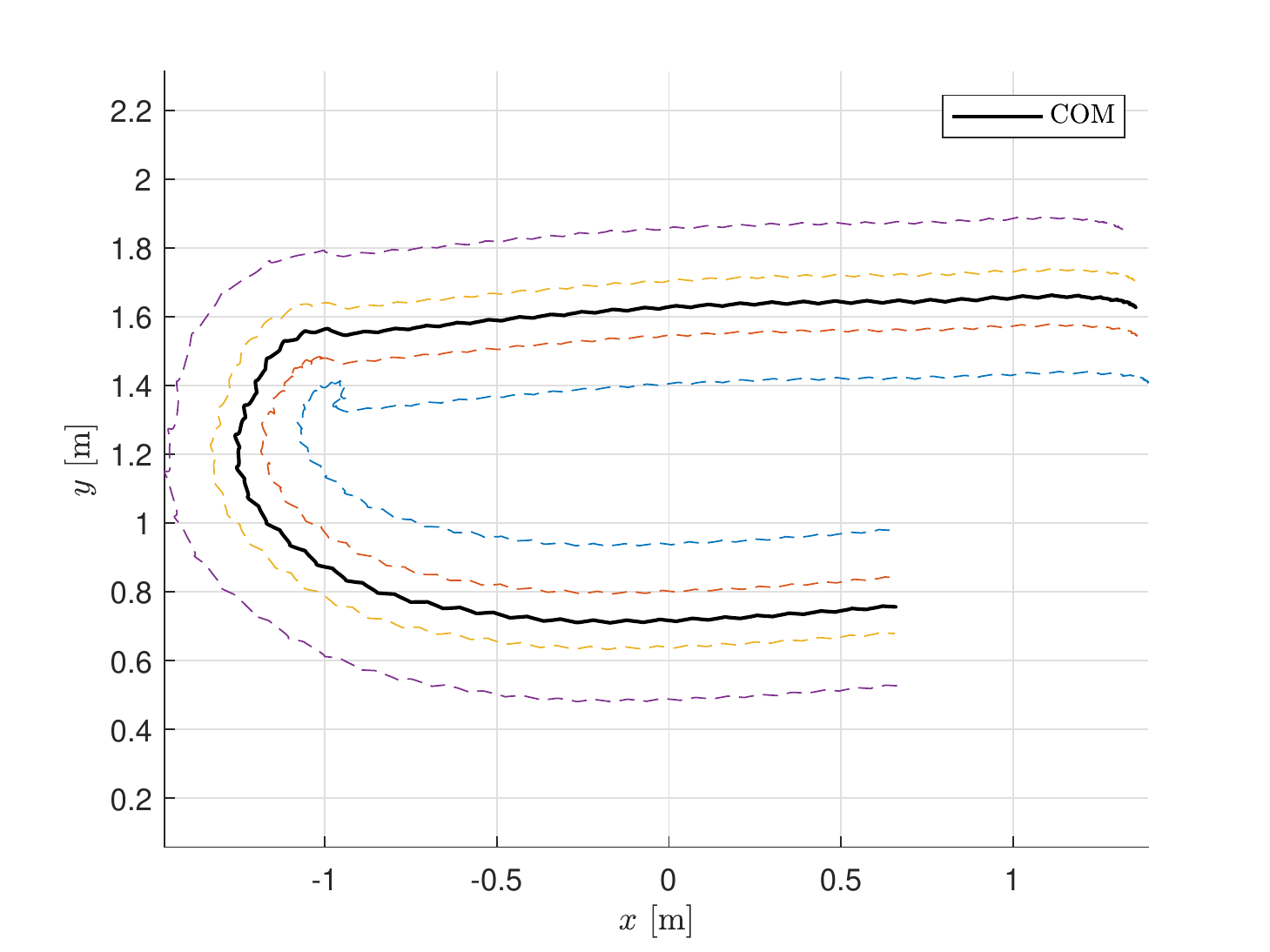}
    \caption{Five boat configuration tracking a yaw of $\pi~\si{rad}$ (left) for $45~\si{s}$, then tracking a yaw of $0~\si{rad}$ (right) for $45~\si{s}$, all while maintaining a velocity of $6~\si{cm/s}$. Each dashed color is an individual module, and the the solid black line is the center of mass.}
    \label{fig:trajectory}
\end{figure}

\begin{figure}[t]
    \centering
    \includegraphics[width=\linewidth, trim={0 0 0 0.5cm}]{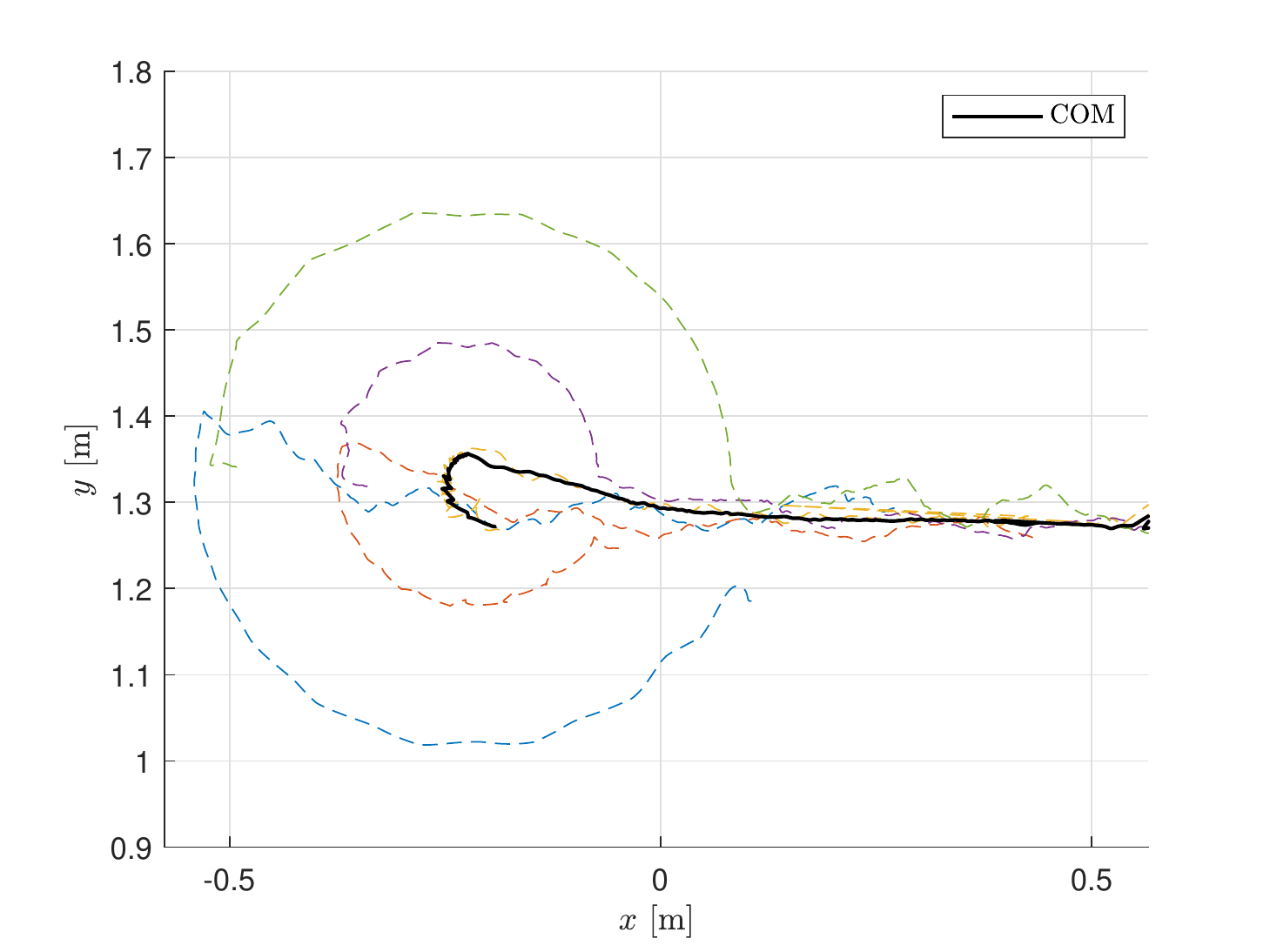}
    \caption{Trajectory of five boat lattice responding to a $\pi~\si{rad}$ step input without a desired velocity. Dotted lines show individual boat trajectories, while the solid line follows the center of mass. The lattice turns in place until the desired yaw is achieved, and then drifts uncontrollably sideways due to the discontinuity in~\eqref{eq:waveformNew}.}
    \label{fig:sidewaysAfterYaw}
    \vspace{-0.5cm}
\end{figure}

\begin{figure}[t]
    \centering
    \subfloat[\label{fig:mismatchedVel}]{\includegraphics[width=\linewidth, trim={0 0 0 0.5cm}]{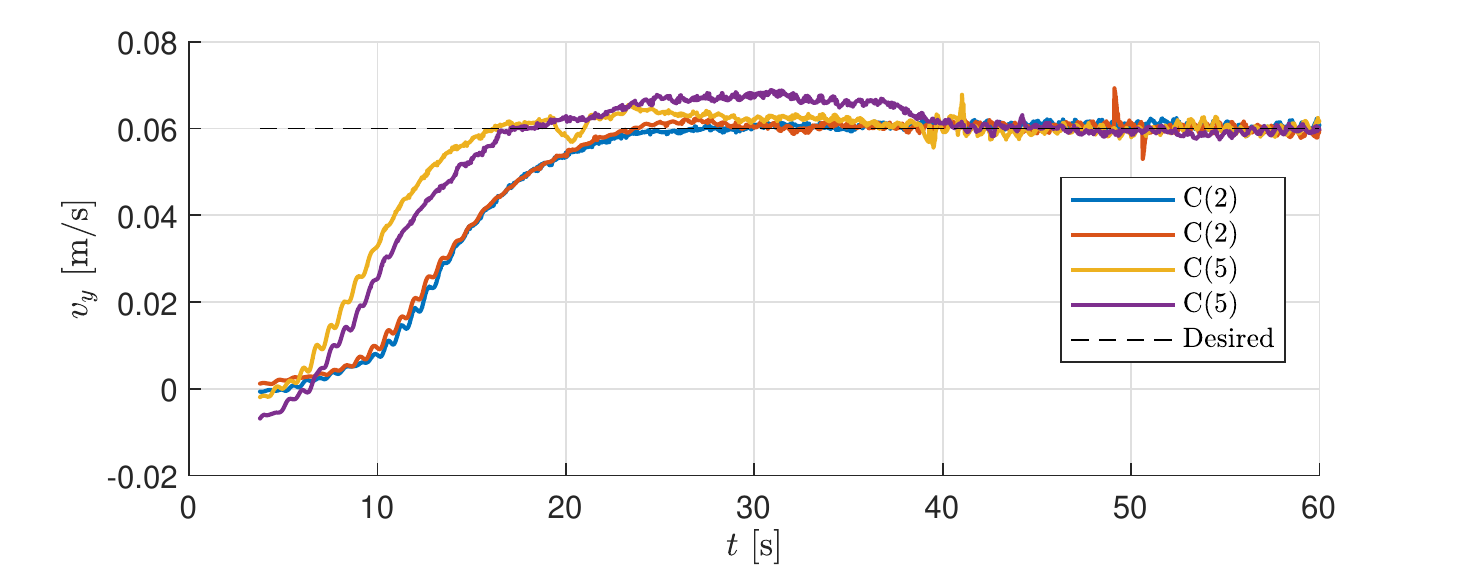}
       }
          \hfill
    \subfloat[\label{fig:mismatchedYaw}]{\includegraphics[width=\linewidth, trim={0 0 0 0.5cm}]{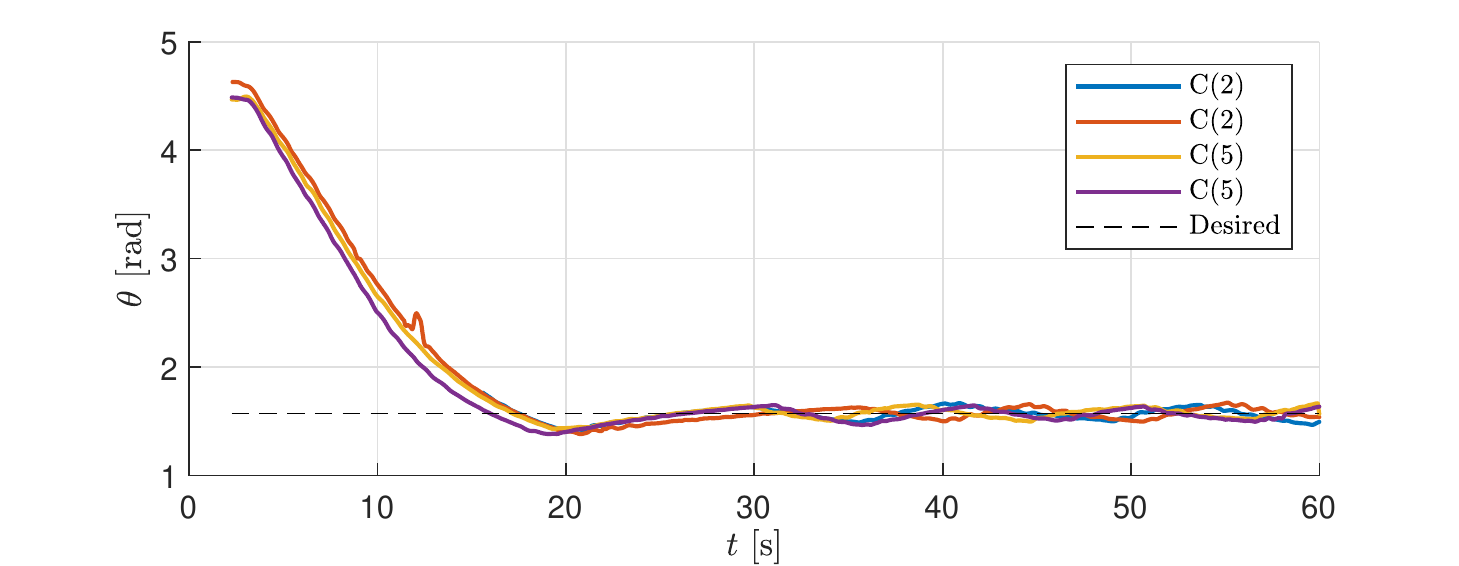}
        }
    \caption{(a) Velocity and (b) yaw step response for the $5$ boat configuration where the drag coefficients are chosen as $C_L(n)$ and $C_R(n)$ with $n=5$ and $n=2$. Two tests are shown for each.}
    \label{fig:mismatched}
    \vspace{-0.5cm}
\end{figure}

%%%%%%%%%%%%%%%%%%%%%%%%%%%%
%%%%%%%%%%%%%%%%%%%%%%%%%%%%
%%%%%%%%%%%%%%%%%%%%%%%%%%%%

\section{Discussion} \label{sec:discussion}

The controller presented in Sec.~\ref{sec:control} is intended to allow parallel lattices of Modboats to be controllable as a single unit. The experimental results presented in Sec.~\ref{sec:experiments} show that this is the case: parallel lattices of Modboats can be driven at a desired surge velocity and to a given yaw angle. With a suitable control law, this performance can easily be extended to waypoint tracking and more complex behaviors. 

Velocity tracking, presented in Fig.~\ref{fig:velResults} and summarized in Table~\ref{tab:performanceQ}, is highly effective regardless of the lattice size. All tested lattices tracked the desired velocity to within $0.3~\si{cm/s}$ and achieve it within $5.7~\si{s}$, which shows that the controller is not sensitive to lattice size. The controller can achieve velocities from $3.0\text{--}8.0~\si{cm/s}$, bounded by the minimum amplitude needed to activate the flippers and the maximum allowable amplitude. While a maximum speed of $0.5$ body-lengths per second (the Modboat diameter is $15.2~\si{cm}$) is not fast, it is sufficient. Velocity tracking performance is comparable during combined yaw/velocity testing for both rise time and RMS error. There is a noticeable drop in velocity during the turning maneuver (see Fig.~\ref{fig:trajectory}), but it recovers similarly once the turn is complete.

Yaw can be tracked very accurately when the lattice is moving. Small adjustments to the amplitude of each boat's oscillation yaws the lattice and tracks the desired heading to within $0.11~\si{rad}$ ($6.5^\circ$), as shown in Table~\ref{tab:performanceQ}. Yaw tracking suffers, however, when no velocity is desired; although the yaw can be driven to the desired value overall, the tracking error increases significantly, as shown in Table~\ref{tab:performanceQ}. This occurs because the lattice attempts to remain in place, but finer adjustments require reversing thrust. If there is enough initial overshoot, significant oscillations around the desired value begin and are not damped  by the derivative term, such as for the two and three boat lattices in Fig.~\ref{fig:yawResults} with  $\pi/2~\si{rad}$ step input and for the four and five boat lattices in Fig.~\ref{fig:mismatchedYaw} with $\pi~\si{rad}$ step input. Otherwise, a significant settling time is seen as the boats overcome the large inertia of the lattice, such as for the four and five boat lattices in Fig.~\ref{fig:yawResults}. Pure yaw control, however, displays reasonable rise times that are not significantly affected by the step input, as shown in Table~\ref{tab:performanceQ}, whereas rise time when swimming is significantly slower.

A secondary problem with controlling yaw alone is shown in Fig.~\ref{fig:sidewaysAfterYaw}; after achieving the desired yaw with minimal COM motion, the lattice drifts sideways. Since our controller can produce forces only along its surge axis, we cannot counteract this. Both issues stem from the same cause: the Modboats' unique propulsive mechanism cannot smoothly transition from forward to reverse. Eq.~\eqref{eq:waveformNew} is discontinuous when $(\phi_0)_i$ changes, especially if $A_i$ is small. The transition from forward to reverse thrust thus creates sideways forces and yaw torques that disrupt the controller.

We also consider the sensitivity of the controller to the drag coefficients. Fig.~\ref{fig:mismatched} shows the controller performance when given incorrect drag coefficients. Behavior for velocity (Fig.~\ref{fig:mismatchedVel}) matches what might be intuited: it takes longer to achieve the desired velocity, but the feedback controller is sufficient since the desired velocity is achievable. The final velocity tracking, which is based on a thrust/drag balance, is comparable to that achieved with the correct drag coefficients. The behavior in yaw (Fig.~\ref{fig:mismatchedYaw}) shows similar behavior, although the transient region is closer to the original behavior. This is likely because the $I\alpha$ term dominates the transient region, minimizing the effect of the mismeasured $C_L$. This behavior is encouraging, as it is likely that non-parallel configurations will be controllable even if their drag coefficients are not explicitly characterized. 

%%%%%%%%%%%%%%%%%%%%%%%%%%%%
%%%%%%%%%%%%%%%%%%%%%%%%%%%%
%%%%%%%%%%%%%%%%%%%%%%%%%%%%

\section{Conclusion}

In this work we have presented a centralized approach that allows parallel lattices of underactuated modules to swim as a single unit. When applied to lattices of Modboats, it minimizes internal forces and guarantees no unintentional undocking. The controller is capable of tracking a desired surge velocity and yaw angle for configurations of various sizes, and has been experimentally verified for lattices consisting of two to five boats. Velocity tracking is shown to be highly effective, and yaw tracking is shown to be effective when the lattice is moving forward.

The controller struggles to track yaw angle while stationary, however, generating oscillations and sideways drift that cannot be counteracted by the given control law. This is poorly suited for docking or station keeping, both of which require precise orientation control while stationary. Future work will consider ways to add control of the lattice's sway axis and reduce the observed oscillations.

A number of approaches for controlling smaller lattices --- i.e. single Modboats --- already exist~\cite{Knizhnik2020a,Knizhnik2021a}. For larger lattices we theorize that our controller will continue to perform well, although problems may arise as the yaw authority of individual boats scales linearly with their distance from the center of mass, but the angular drag and inertia scale quadratically. Velocity tracking is likely to remain effective, but our testing tank is too small for larger configurations. 

Finally, although we have considered only parallel lattices in this work, the control structure does not explicitly contain this restriction. Minor adjustments should allow the controller to function effectively for arbitrary lattices. Future work will consider the behavior of this controller for non-parallel lattices while still guaranteeing no unintentional undocking. We will also explore additional ways to avoid undocking that allow more flexible control inputs, which should increase the capabilities of the final lattice. Finally, we hope to consider the robustness of this controller to disturbances such as are commonly found in the ocean.

%%%%%%%%%%%%%%%%%%%%%%%%%%%%
%%%%%%%%%%%%%%%%%%%%%%%%%%%%
%%%%%%%%%%%%%%%%%%%%%%%%%%%%

\section*{Acknowledgment}

We thank Dr. M. Ani Hsieh for the use of her instrumented water basin in obtaining all of the testing data.

\bibliographystyle{./bibliography/IEEEtran}
\bibliography{./bibliography/IEEEabrv,./bibliography/iros2020,./bibliography/nonpaper,./bibliography/references}

% Generated by IEEEtran.bst, version: 1.14 (2015/08/26)
\begin{thebibliography}{10}
\providecommand{\url}[1]{#1}
\csname url@samestyle\endcsname
\providecommand{\newblock}{\relax}
\providecommand{\bibinfo}[2]{#2}
\providecommand{\BIBentrySTDinterwordspacing}{\spaceskip=0pt\relax}
\providecommand{\BIBentryALTinterwordstretchfactor}{4}
\providecommand{\BIBentryALTinterwordspacing}{\spaceskip=\fontdimen2\font plus
\BIBentryALTinterwordstretchfactor\fontdimen3\font minus
  \fontdimen4\font\relax}
\providecommand{\BIBforeignlanguage}[2]{{%
\expandafter\ifx\csname l@#1\endcsname\relax
\typeout{** WARNING: IEEEtran.bst: No hyphenation pattern has been}%
\typeout{** loaded for the language `#1'. Using the pattern for}%
\typeout{** the default language instead.}%
\else
\language=\csname l@#1\endcsname
\fi
#2}}
\providecommand{\BIBdecl}{\relax}
\BIBdecl

\bibitem{Paulos2015}
J.~Paulos, N.~Eckenstein, T.~Tosun, J.~Seo, J.~Davey, J.~Greco, V.~Kumar, and
  M.~Yim, ``{Automated Self-Assembly of Large Maritime Structures by a Team of
  Robotic Boats},'' \emph{IEEE Transactions on Automation Science and
  Engineering}, vol.~12, no.~3, pp. 958--968, 2015.

\bibitem{Vasilescu2010AMOURPayloads}
I.~Vasilescu, C.~Detweiler, M.~Doniec, D.~Gurdan, S.~Sosnowski, J.~Stumpf, and
  D.~Rus, ``{AMOUR V: A Hovering Energy Efficient Underwater Robot Capable of
  Dynamic Payloads},'' \emph{The International Journal of Robotics Research},
  vol.~29, no.~5, pp. 547--570, 2010.

\bibitem{Mintchev2014}
S.~Mintchev, E.~Donati, S.~Marrazza, and C.~Stefanini, ``{Mechatronic design of
  a miniature underwater robot for swarm operations},'' in \emph{2014 IEEE
  International Conference on Robotics and Automation (ICRA)}, Hong Kong, 2014,
  pp. 2938--2943.

\bibitem{Mintchev2014TowardsRobots}
S.~Mintchev, R.~Ranzani, F.~Fabiani, and C.~Stefanini, ``{Towards docking for
  small scale underwater robots},'' \emph{Autonomous Robots}, vol.~38, no.~3,
  pp. 283--299, 8 2014.

\bibitem{Seo2013AssemblyRobots}
J.~Seo, M.~Yim, and V.~Kumar, ``{Assembly planning for planar structures of a
  brick wall pattern with rectangular modular robots},'' in \emph{IEEE
  International Conference on Automation Science and Engineering}, 2013, pp.
  1016--1021.

\bibitem{Seo2016AssemblyModules}
J.~Seo, M.~Yim, and V.~Kumar, ``{Assembly sequence planning for constructing
  planar structures with rectangular modules},'' in \emph{2016 IEEE
  International Conference on Robotics and Automation (ICRA)}, 6 2016, pp.
  5477--5482.

\bibitem{OHara2014}
I.~O'Hara, J.~Paulos, J.~Davey, N.~Eckenstein, N.~Doshi, T.~Tosun, J.~Greco,
  J.~Seo, M.~Turpin, V.~Kumar, and M.~Yim, ``{Self-assembly of a swarm of
  autonomous boats into floating structures},'' in \emph{2014 IEEE
  International Conference on Robotics and Automation (ICRA)}, Hong Kong, 2014,
  pp. 1234--1240.

\bibitem{Wang2018DesignVehicle}
W.~Wang, L.~A. Mateos, S.~Park, P.~Leoni, B.~Gheneti, F.~Duarte, C.~Ratti, and
  D.~Rus, ``{Design, Modeling, and Nonlinear Model Predictive Tracking Control
  of a Novel Autonomous Surface Vehicle},'' in \emph{2018 IEEE International
  Conference on Robotics and Automation (ICRA)}, Brisbane, Australia, 2018, pp.
  6189--6196.

\bibitem{Wang2020RoboatEnvironments}
W.~Wang, T.~Shan, P.~Leoni, D.~Fern{\'{a}}ndez-Guti{\'{e}}rrez, D.~Meyers,
  C.~Ratti, and D.~Rus, ``{Roboat II: A Novel Autonomous Surface Vessel for
  Urban Environments},'' in \emph{2020 IEEE/RSJ International Conference on
  Intelligent Robots and Systems (IROS)}, Las Vegas, NV (Virtual), 2020, pp.
  1740--1747.

\bibitem{Wang2020DistributedVessels}
W.~Wang, Z.~Wang, L.~Mateos, K.~W. Huang, M.~Schwager, C.~Ratti, and D.~Rus,
  ``{Distributed motion control for multiple connected surface vessels},'' in
  \emph{2020 IEEE/RSJ International Conference on Intelligent Robots and
  Systems (IROS)}.\hskip 1em plus 0.5em minus 0.4em\relax Las Vegas, NV
  (Virtual): Institute of Electrical and Electronics Engineers Inc., 10 2020,
  pp. 11\,658--11\,665.

\bibitem{Park2019CoordinatedApproach}
S.~Park, E.~Kayacan, C.~Ratti, and D.~Rus, ``{Coordinated control of a
  reconfigurable multi-vessel platform: Robust control approach},'' in
  \emph{2019 IEEE International Conference on Robotics and Automation (ICRA)},
  5 2019, pp. 4633--4639.

\bibitem{Knizhnik2021}
G.~Knizhnik and M.~Yim, ``{Docking and Undocking a Modular Underactuated
  Oscillating Swimming Robot},'' in \emph{2021 IEEE International Conference on
  Robotics and Automation (ICRA)}, Xi'an, China, 2021, pp. 6754--6760.

\bibitem{modboatsOnline}
G.~Knizhnik, ``Modboat: A single-motor modular self-reconfigurable robot,'' Feb
  2022, \url{https://www.modlabupenn.org/modboats/}.

\bibitem{Knizhnik2020a}
G.~Knizhnik, P.~Dezonia, and M.~Yim, ``{Pauses Provide Effective Control for an
  Underactuated Oscillating Swimming Robot},'' \emph{IEEE Robotics and
  Automation Letters}, vol.~5, no.~4, pp. 5075--5080, 10 2020.

\bibitem{Gabrich2020ModQuad-DoF:Quadrotors}
B.~Gabrich, G.~Li, and M.~Yim, ``{ModQuad-DoF: A Novel Yaw Actuation for
  Modular Quadrotors},'' in \emph{2020 IEEE International Conference on
  Robotics and Automation (ICRA)}, Paris, France, 5 2020, pp. 8267--8273.

\bibitem{Knizhnik2021a}
G.~Knizhnik and M.~Yim, ``{Thrust Direction Control of an Underactuated
  Oscillating Swimming Robot},'' in \emph{2021 IEEE/RSJ International
  Conference on Intelligent Robots and Systems (IROS)}, Prague, Czech Republic
  (Virtual), 2021, pp. 8665--8670.

\end{thebibliography}

\end{document}